\documentclass[letterpaper]{article} 
\usepackage{mysty}  
\usepackage{times}  
\usepackage{helvet}  
\usepackage{courier}  
\usepackage[hyphens]{url}  
\usepackage{graphicx} 
\usepackage{natbib}  
\usepackage{caption} 
\usepackage{algorithm}
\usepackage{algorithmic}

\usepackage{amsmath}
\usepackage{amssymb}
\usepackage{amsthm}
\usepackage{mathtools}
\usepackage{amsfonts}
\usepackage{bm}
\mathtoolsset{showonlyrefs,showmanualtags}

\newtheorem{definition}{Definition}
\newtheorem{theorem}{Theorem}

\newcommand\rank[1]{\mathrm{\texttt{rank}}(#1)}
\newcommand\dist[1]{\mathrm{\texttt{dist}}(#1)}

\newcommand{\bmx}{\bm{x}}
\newcommand{\bmy}{\bm{y}}
\newcommand{\bmf}{\bm{f}}
\newcommand{\bmz}{\bm{z}}

\usepackage{newfloat}
\usepackage{listings}
\DeclareCaptionStyle{ruled}{labelfont=normalfont,labelsep=colon,strut=off} 
\floatstyle{ruled}
\newfloat{listing}{tb}{lst}{}
\floatname{listing}{Listing}
\title{Towards Running Time Analysis of \\Interactive Multi-objective Evolutionary Algorithms}
\author{
    Tianhao Lu,
    Chao Bian and Chao Qian\thanks{Chao Qian is the corresponding author.}
}
\affiliations{
    State Key Laboratory for Novel Software Technology, Nanjing University, Nanjing 210023, China\\

    \{luth, bianc, qianc\}@lamda.nju.edu.cn
}

\usepackage{bibentry}

\begin{document}

\maketitle

\begin{abstract}
    Evolutionary algorithms (EAs) are widely used for multi-objective optimization due to their population-based nature. Traditional multi-objective EAs (MOEAs) generate a large set of solutions to approximate the Pareto front, leaving a decision maker (DM) with the task of selecting a preferred solution. However, this process can be inefficient and time-consuming, especially when there are many objectives or the subjective preferences of DM is known. To address this issue, interactive MOEAs (iMOEAs) combine decision making into the optimization process, i.e., update the population with the help of the DM. In contrast to their wide applications, there has existed only two pieces of theoretical works on iMOEAs, which only considered interactive variants of the two simple single-objective algorithms, RLS and (1+1)-EA. This paper provides the first running time analysis (the essential theoretical aspect of EAs) for practical iMOEAs. Specifically, we prove that the expected running time of the well-developed interactive NSGA-II (called R-NSGA-II) for solving the OneMinMax and OneJumpZeroJump problems is $O(n \log n)$ and $O(n^k)$, respectively, which are all asymptotically faster than the traditional NSGA-II. Meanwhile, we present a variant of OneMinMax, and prove that R-NSGA-II can be exponentially slower than NSGA-II. These results provide theoretical justification for the effectiveness of iMOEAs while identifying situations where they may fail. Experiments are also conducted to validate the theoretical results.
\end{abstract}

\section{Introduction}
Multi-objective optimization (MOP)~\cite{mop} requires to optimize several objective functions simultaneously, which appears in various real-world applications. Since these functions usually conflict with each other, there does not exist a solution which can outperform any other solution in every objective function. For example, in neural architecture search~\cite{debneural}, an architecture with higher accuracy often has higher complexity as well. Therefore, the goal of MOP is to identify a set of solutions that represent diverse optimal trade-offs between objective functions, referred to as Pareto optimal solutions.

Evolutionary algorithms (EAs)~\cite{back1996evolutionary}  are a large class of randomized heuristic optimization algorithms, inspired by natural evolution. They keep a set of solutions (called population) and iteratively improve it by generating offspring solutions and selecting better ones. The population-based nature makes EAs highly appropriate for solving MOPs, which have become a popular tool in various real-world MOP situations~\cite{deb2002fast,zhang,beume2007sms}.

After the optimization of multi-objective EAs (MOEAs), a decision maker (DM) needs to select one solution from the generated set of non-dominated solutions based on her (or his) preference. However, the paradigm of finding a large set of solutions at first and then selecting an appropriate one may be inefficient and time-consuming under some circumstances, especially when the number of objectives is large or DM has some subjective preference information. Specifically, when the number of objectives increases, the objective space gets increasingly complex, making it hard to find solutions with good convergence and spread; when the subjective preference of DM is known, such information can actually give a directional bias for the search, while MOEAs  usually ignore it and simply search the whole solution space equally, resulting in the waste of computing resources in inessential regions.

Interactive MOEAs (iMOEAs)~\cite{tanino1993interactive} is thus proposed, which combines decision making into the optimization process. That is, in the population update procedure of each generation, the DM is used to select the preferred solution out of two incomparable solutions. The preference information provided by the DM can be a utility function~\cite{gong2011interactive, koksalan2010interactive, battiti2010brain}, which outputs a function value to indicate the goodness of a solution based on the solution's objective vector; a reference point~\cite{huber2015simulation, said2010r, thiele2009preference}, which represents the DM's favorite objective vector and prefers a solution with objective vector closer to it; and a surrogate model~\cite{li2019progressive}, which directly selects a preferred solution by using a machine learning model, e.g., neural network. 
Due to their fast convergence and the ability to reflect the DM's preference, iMOEAs have achieved many successful real-world applications, e.g., aircraft design~\cite{bandte2004broad, weigang2007interactive}, software defign~\cite{RAMIREZ201892}, and land use planning~\cite{janssen2008multiobjective}.

In contrast to wide applications, the theoretical foundation of MOEAs as well as iMOEAs is still underdeveloped, which is mainly due to the sophisticated behaviors of EAs and the hardness of MOPs. The running time analysis, the essential theoretical aspect, of MOEAs starts from analyzing the expected running time of simple MOEAs (e.g., SEMO (or GSEMO) which employs the one-bit (or bit-wise) mutation operator to generate an offspring solution in each iteration and keeps the non-dominated solutions generated-so-far in the population) for solving multi-objective synthetic and combinatorial optimization problems~\cite{giel2003expected,laumanns-nc04-knapsack,LaumannsTEC04,Neumann07,Horoba09,Giel10,Neumann10,doerr2013lower,Qian13,bian2018general,dinot2023runtime}.

Recently, researchers are shifting their focus to analyzing practical MOEAs.
Zheng, Liu and Doerr~\shortcite{zheng2021first} analyzed the expected running time of the popular NSGA-II~\cite{deb2002fast} solving the OneMinMax and LOTZ problems for the first time. Later, Zheng and Doerr~\shortcite{zheng2022better} considered a modified crowding distance method which recalculates the crowding distance each time the solution with the smallest crowding distance is deleted, and proved that it can lead to better approximation to the Pareto front. 
Bian and Qian~\shortcite{bian2022running} proved that replacing binary tournament selection with stochastic tournament selection in parent selection can significantly decrease the running time of NSGA-II. More works include the analysis of NSGA-II solving the many-objective problem mOneMinMax~\cite{zheng2023runtime} and the bi-objective minimum spanning tree problem~\cite{cerf2023first}, NSGA-II solving the LOTZ problem in the presence of noise~\cite{dang2023analysing}, the effectiveness of the crossover operator~\cite{dang2023proof,doerr2023runtime}, and the lower bound of NSGA-II solving the OneMinMax and OneJumpZeroJump problems~\cite{doerr2023understanding}.
Besides, the expected running time of other popular MOEAs has also been considered, e.g., MOEA/D~\cite{huang2019running,huang2020runtime,huang2021runtime}, SMS-EMOA~\cite{bian2023sms}, and NSGA-III~\cite{doerr2022mathematical}.

Though much effort has been devoted to analyzing MOEAs, the theoretical analysis of iMOEAs is quite rare. To the best of our knowledge, there have been only two pieces of work. Brockhoff et al.~\shortcite{brockhoff2012runtime} analyzed the expected running time of iRLS and (1+1)-iEA solving two bi-objective problems LOTZ and COCZ, where iRLS and (1+1)-iEA maintain exactly one solution in the population and iteratively improve the solution by generating one offspring solution and keeping the better one. Specifically, the parent solution and offspring solution are first compared according to the Pareto dominance relationship, i.e., the dominating one will be picked; if they are incomparable, the DM will pick the preferred one, i.e., the one with a larger utility function value will be picked. They proved that the performance of (1+1)-iEA and iRLS may dramatically worsen if the utility function is non-linear instead of linear. Neumann and Nguyen~\shortcite{neumann2014impact} analyzed the expected running time of iRLS and (1+1)-iEA solving the multi-objective knapsack problem, and proved that using a linear utility function as DM can make iRLS and (1+1)-iEA perform as same as their single-objective counterparts, i.e., RLS and (1+1)-EA~\cite{droste2002analysis}, solving the weighted objective sum of knapsack. Though the above two works analyzed (1+1)-iEA and iRLS solving multi-objective problems, the two algorithms are actually the variants of the simple single-objective algorithms RLS and (1+1)-EA, which cannot reflect the common structure of iMOEAs.

In this paper, we make the first attempt towards analyzing the expected running time of practical iMOEAs. We consider R-NSGA-II~\cite{deb2006reference}, which is an interactive variant of the well-known NSGA-II~\cite{deb2002fast}, and prefers a solution with a smaller distance to a reference point based on the DM’s preference, instead of selecting from incomparable solutions based on their crowding distance. We prove that the expected running time of R-NSGA-II for solving OneMinMax and OneJumpZeroJump is $O(Nn \log n)$ and $O(n^k)$, respectively, where $N$ is the population size, $n$ is the problem size, and $k$ is the parameter of OneJumpZeroJump. Note that OneMinMax and OneJumpZeroJump are two commonly used bi-objective optimization problems in theoretical analyses of MOEAs~\cite{zheng2021first,zheng2023runtime,doerr2023understanding,bian2023sms}. Here solving a problem means finding a solution corresponding to the reference point, which is what the DM really wants. We also prove that the expected running time of NSGA-II solving these two problems is $O(Nn\log n)$ and $O(Nn^k)$, respectively, where the population size $N$ is required to be $\Omega(n)$. Thus, R-NSGA-II using a constant population size can be asymptotically faster than NSGA-II. The main reason is that R-NSGA-II can always keep the solution closest to the reference point in the population, while NSGA-II requires a large population size to ensure this. These results provide theoretical justification for the effectiveness of iMOEAs, and are also verified by experiments.

Next, we present a variant of OneMinMax, and prove that the expected running time of NSGA-II with $N \geq 4(n+1)$ is $O(Nn\log n)$, while R-NSGA-II using any polynomial population size requires $\Omega(n^{n/4})$. Thus, R-NSGA-II can be exponentially slower than NSGA-II. This is because the search direction of R-NSGA-II is guided towards the reference point in the objective space, but the solutions are actually getting far away from the solution corresponding to the reference point in the decision space. This finding may generally hold in complex real-world problems with rugged fitness landscape, e.g., we have verified it by the experiments on the combinatorial problem multi-objective NK-landscape~\cite{aguirre2007working}. Thus, our theoretical results suggest that when designing practical iMOEAs, relying solely on DM may be risky, while maintaining some diversity in the population may be helpful.

\section{Preliminaries}
In this section, we first introduce multi-objective optimization, and then present the procedures of NSGA-II and its interactive variant, R-NSGA-II.

\subsection{Multi-objective Optimization}

Multi-objective optimization entails the simultaneous optimization of two or more objective functions, as illustrated in Definition 1. In this paper, we focus on maximization, and minimization can be defined similarly. The objectives often exhibit conflicting interests, resulting in the absence of a complete order within the solution space $S$. Consequently, comparing solutions often relies on the domination relationship, as presented in Definition 2. A solution is deemed Pareto optimal if no other solution in $S$ dominates it. The collection of objective vectors of all the Pareto optimal solutions is called the Pareto front. The primary aim of multi-objective optimization is to identify the Pareto front, i.e., find at least one corresponding solution for each objective vector within this front. In this paper, we will study pseudo-Boolean problems, i.e., the solution space $S=\{0,1\}^n$, where $n$ denotes the problem size.

\begin{definition}[Multi-objective Optimization]\label{def_MO}
Given a feasible solution space $S$ and objective functions $f_1, f_2, ..., f_m$, multi-objective optimization can be formulated as
\begin{align}
		\max\nolimits_{\bmx \in
			S}\; \big(f_1(\bmx),f_2(\bmx),...,f_m(\bmx)\big).
\end{align}
\end{definition}

\begin{definition}[Dominance]\label{def_DO}
    Let $\bmf = (f_1, f_2,..., f_m):S\rightarrow \mathbb{R}^m$ be the objective vector. For two solutions $\bmx$ and $\bmy\in S:$
    \begin{itemize}
        \item $\bmx$ \emph{weakly dominates }$\bmy$ (denoted as $\bmx\succeq \bmy$) if $\forall 1\leq i\leq m,$ $f_i(\bmx)\geq f_i(\bmy)$;
        \item $\bmx$ \emph{dominates } $\bmy$ (denoted as $\bmx\succ \bmy$) if $\bmx \succeq \bmy$ and $f_i(\bmx)>f_i(\bmy)$ for some $i$;
        \item $\bmx$ and $\bmy$ are \emph{incomparable} if neither $\bmx\succeq \bmy$ nor $\bmy \succeq \bmx$.
    \end{itemize}
\end{definition}

\subsection{NSGA-II}

The NSGA-II algorithm~\cite{deb2002fast} as presented in Algorithm~\ref{alg:nsgaii} is a widely-used MOEA known for its incorporation of two substantial features: non-dominated sorting (presented in Algorithm~\ref{alg:fastsort}) and crowding distance (presented in Algorithm~\ref{alg:crowdist}). NSGA-II commences by initializing a population of $N$ random solutions (line~1). In each generation, it applies bit-wise mutation on each solution in the current population $P$ to generate $N$ offspring solutions (lines~4--7). The bit-wise mutation operator flips each bit of an individual's binary representation independently with a predefined probability, which is $1/n$ in this paper.
 
After generating $N$ offspring solutions, the next generation's population is determined by selecting the best $N$ solutions from the current population $P$ and the offspring population $Q$ (lines~8--14). These solutions are partitioned into non-dominated sets, denoted as $F_1, F_2$, and so on. $F_1$ contains all the non-dominated solutions in the combined set $P\cup Q$, while $F_i$ ($i\geq 2$) contains all the non-dominated solutions in the set $(P\cup Q)\setminus \cup_{j=1}^{i-1} F_j$. The solutions from sets $F_1, F_2$, and so forth, are sequentially added to the next population (lines~10--12) until the population size reaches or exceeds $N$. When the inclusion of a critical set $F_i$ threatens to exceed the population size limit $N$, the crowding distance for each solution in that set is computed. Intuitively, the crowding distance of a solution represents the average distance to its neighbouring solutions in the objective space, 

\begin{algorithm}[!t]
 	\caption{NSGA-II Algorithm~\cite{deb2002fast}}
 	\label{alg:nsgaii}
 	\begin{flushleft}
 		\textbf{Input}: objective functions $f_1,f_2\ldots,f_m$, population size $N$\\
 		\textbf{Output}: $N$ solutions from $\{0,1\}^n$\\
 		\textbf{Process}:
 	\end{flushleft}
 	\begin{algorithmic}[1]
 		\STATE $P \leftarrow N$ solutions randomly selected from $\{0, 1\}^{n}$;
 		\WHILE{criterion is not met}
 		\STATE $Q=\emptyset$;  
 		\FOR{each solutions $\bm x$ in $P$}
 		\STATE apply bit-wise mutation on $\bm x$ to generate $\bm x'$;
 		\STATE add $\bm x'$ into $Q$
 		\ENDFOR
 		\STATE apply Algorithm~\ref{alg:fastsort} to partition $P\cup Q$ into non-dominated sets $F_1,F_2,\ldots$;
 		\STATE let $P=\emptyset$, $i=1$;
 		\WHILE{$|P\cup F_i|<N$}
 		\STATE $P=P\cup F_i$, $i=i+1$ 
 		\ENDWHILE
 		\STATE  apply Algorithm~\ref{alg:crowdist} to assign each solution in $F_i$ with a crowding distance; 
 		\STATE sort the solutions in $F_i$ by crowding distance in descending order, and add the first $N-|P |$ solutions into $P$
 		\ENDWHILE
 		\RETURN $P$
 	\end{algorithmic}
 \end{algorithm}
\vspace{-2cm}
\begin{algorithm}[H]
	\caption{Fast Non-dominated Sorting}
	\label{alg:fastsort}
	\begin{flushleft}
		\textbf{Input}:  a population $P$\\
		\textbf{Output}: non-dominated sets $F_1,F_2,\ldots$\\
		\textbf{Process}:
	\end{flushleft}
	\begin{algorithmic}[1]
		\STATE $F_1=\emptyset$;
		\FOR{ each $\bmx \in P$}
		\STATE $S_{\bmx}=\emptyset ; n_{\bmx}=0$;
		\FOR{ each $y \in P$}
		\IF{$\bmx\succ \bmy$}
		\STATE $S_{\bmx}=S_{\bmx} \cup\{\bmy\}$
		\ELSIF {$\bmx \prec \bmy$}
		\STATE $n_{\bmx}=n_{\bmx}+1$
		\ENDIF
		\ENDFOR
		\IF{$n_{\bmx}=0$}
		\STATE $\rank{\bmx}=1 ; F_{1}=F_{1}\cup \{\bmx\}$
		\ENDIF
		\ENDFOR
		\STATE $i=1$;
		\WHILE{ $F_{i} \neq \emptyset$}
		\STATE $Q=\emptyset$;
		\FOR{ each $\bmx \in F_{i}$ }
		\FOR{ each $\bmy \in S_{\bmx}$ }
		\STATE $n_{\bmy}=n_{\bmy}-1$;
		\IF{ $n_{\bmy}=0$}
		\STATE $\rank{\bmy}=i+1$; $Q=Q \cup\{\bmy\}$
		\ENDIF
		\ENDFOR
		\ENDFOR
		\STATE $i=i+1$; $F_i=Q$
		\ENDWHILE
	\end{algorithmic}
\end{algorithm}

\begin{algorithm}[t]
	\caption{Crowding Distance Assignment}
	\label{alg:crowdist}
	\begin{flushleft}
		\textbf{Input}: $Q=\{\bmx^1,\bmx^2,\ldots,\bmx^l\}$ with the same rank\\
		\textbf{Output}: the crowding distance $\dist{\cdot}$ for each solution in $Q$\\
		\textbf{Process}:
	\end{flushleft}
	\begin{algorithmic}[1]
		\STATE let $\dist{\bmx^j}=0$ for any $j\in\{1,2,\ldots,l\}$;
		\FOR{$i=1$ to $m$}
		\STATE sort the solutions in $Q$ w.r.t. $f_i$ in ascending order;
		\STATE $\dist{Q[1]}=\dist{Q[l]}=\infty$;
		\FOR{$j=2$ to $l-1$}
		\STATE $\dist{Q[j]}\!=\!\dist{Q[j]}+\frac{f_i(Q[j+1])-f_i(Q[j-1])}{f_i(Q[l])-f_i(Q[1])}$
		\ENDFOR
		\ENDFOR
	\end{algorithmic}
\end{algorithm}

\noindent and a larger crowding distance implies a larger diversity of the solution.
The solutions in the critical set $F_i$ with larger crowding distances are selectively included to fill the remaining population slots (line~14).

\subsection{R-NSGA-II}

NSGA-II, as the most widely applied MOEA, has been modified many times for designing its interactive variants, most of which replace the crowding distance with the DM's preference information, e.g., characterized by a reference point~\cite{deb2006reference}, reference direction combining preference thresholds~\cite{deb2007light}, or Gaussian functions~\cite{narukawa2014evolutionary}. Because most interactive variants of NSGA-II adopt a reference point as the DM, we consider R-NSGA-II~\cite{deb2006reference}, i.e., Reference-point-based NSGA-II, in this paper, which is widely applied in practice.

R-NSGA-II is similar to NSGA-II, except that the solution selection from the critical set $F_i$ depends on the distance to the reference point instead of the traditional crowding distance. A solution having a smaller distance to the reference point is favored. Specifically, ``a crowding distance" in line~13 of Algorithm~\ref{alg:nsgaii} changes to ``the distance to the reference point", and ``crowding distance in descending order" in line~14 changes to ``the distance to the reference point in ascending order". Note that the distance to the reference point denotes the Euclidean distance in the objective space.  

The selection of reference point depends on the specific problem to be solved, and reflects the DM’s preference. For ease of theoretical analysis, we use an objective vector in the Pareto front as a reference point, following previous theoretical analysis of iMOEAs~\cite{brockhoff2012runtime,neumann2014impact}. As the reference point reflects the DM’s preference, the stopping criterion of an algorithm is set as finding the reference point, i.e., finding at least one solution whose objective vector equals to the reference point. Such a natural setting can avoid any reliance on subjective human factors and ensure objectivity for theoretical analysis.

\section{R-NSGA-II Can Do better Than NSGA-II}
In this section, we analyze the expected running time of NSGA-II and R-NSGA-II for solving two bi-objective problems, i.e., OneMinMax and OneJumpZeroJump. We prove that for these two problems, R-NSGA-II can find the reference point asymptotically faster than NSGA-II.

\subsection{Analysis on OneMinMax}
The OneMinMax problem as presented in Definition 1 aims to maximize the number of 0-bits and the number of 1-bits of a binary bit string simultaneously. The Pareto front of OneMinMax is $F^* = \{(k, n-k)\mid k\in[0..n]\}$, and any solution $\bmx\in\{0,1\}^n$ is Pareto optimal, corresponding to the objective vector $(n-|\bmx|_1,|\bmx|_1)$ in the Pareto front, where $|\cdot|_1$ denotes the number of 1-bits of a solution. The size of the Pareto front $|F^*|=n+1$. We let the reference point $\bmf(\bmz)$ be any vector $(k,n-k),k\in[0..n]$, since the selection of reference point will not change our results.

\begin{definition}[OneMinMax~\cite{Giel10}]
    The OneMinMax problem of size $n$ is to find $n$-bits binary strings which maximize
    \begin{align}
        \bmf(\bmx)=\Big(n-\sum\nolimits_{i=1}^{n}x_i, \sum\nolimits_{i=1}^{n}x_i\Big)
    \end{align}
    where $x_i$ denotes the $i$-th bit of $\bmx\in \{0,1\}^{n}$.
\end{definition}

We prove in Theorems~\ref{theo-omm-nsgaii} and~\ref{theo-omm-rnsgaii} that NSGA-II can find the reference point in $O(Nn\log n)$ expected number of fitness evaluations, where the population size $N\ge 4(n+1)$; and R-NSGA-II can find the reference point in $O(Nn\log n)$ expected time for any population size $N$. Thus, our results show that if using a constant population size, R-NSGA-II can be at least $n$ times faster than NSGA-II with $N\ge 4(n+1)$. The main proof idea of Theorem~\ref{theo-omm-nsgaii} is that the population of NSGA-II can continuously expand in the Pareto front, until the reference point is found.

\begin{theorem}\label{theo-omm-nsgaii}
For NSGA-II solving the OneMinMax problem, if using a population size $N$ of at least $4(n+1)$, then the expected number of fitness evaluations for finding the reference point $(k, n-k)$, $k\in[0..n]$, is $O(Nn\log n)$.
\end{theorem}
\begin{proof}
Note that all the solutions are Pareto optimal, thus any solution would be in $F_1$ after non-dominated sorting. Since the population size $N\geq 4(n+1)$, each objective vector in the Pareto front will not be lost once found, according to Lemma 1 in~\cite{zheng2021first}.

We let the reference point $\bmf(\bmz)=(n-k,k)$ for some $k\in[0..n]$. That is, $\bmz$ is a $n$ bits binary string with $|\bmz|_0=n-k$ and $|\bmz|_1=k$. Let $\bmx_t^*$ be the solution closest to $z$, i.e., $\bmx_t^*\in \arg\min_{\bmx\in P_t} d(\bmf(\bmx),\bmf(\bmz))$. Note that $\bmx_t^*$ will be selected as the parent solution and used to generate an offspring solution $\bm{x}'$ by bit-wise mutation. If $|\bmx_t^*|>k$, then flipping one of the 1-bits in $\bmx_t^*$ can generate $\bmx'$ such that $d(\bmf(\bmx'), \bmf(\bmz))< d(\bmf(\bmx_{t}^*), \bmf(\bmz))$, whose probability is $\frac{|\bmx_{t}^{*}|_1}{n}\cdot (1-\frac{1}{n})^{n-1}\geq \frac{|\bmx_{t}^{*}|_1}{en}$. Once $\bmx'$ is generated, there must exist one solution with the same objective vector as $\bmx'$ in the next population by Lemma 1 in \cite{zheng2021first}, implying that the closet distance to $\bmf(\bmz)$ can decrease by 1. The similar analysis also holds for the case $|\bmx_t^*|<k$. Note that the closet distance to $\bmf(\bmz)$ can decrease at most $n$ times, thus the expected number of generations for finding the reference point is at most $\sum_{i=1}^n en/i=O(n\log n)$, implying an upper bound $O(Nn\log n)$ on the expected number of fitness evaluations.
\end{proof}

The proof of Theorem~\ref{theo-omm-rnsgaii} is similar to that of Theorem~\ref{theo-omm-nsgaii}. The main difference is that R-NSGA-II relies on the distance to the reference point for population update, which ensures that the solution closest to the reference point can be maintained in the population.
\begin{theorem}\label{theo-omm-rnsgaii}
    For R-NSGA-II with any population size $N$ solving the OneMinMax problem, the expected number of fitness evaluations for finding the reference point $(k, n-k)$, $k\in[0..n]$, is $O(Nn\log n)$.
\end{theorem}
\begin{proof}
    Based on the proof of Theorem~1, we find that the key point is to keep decreasing the closet distance to the reference point until the reference point is found. In NSGA-II, after non-dominated sorting, the algorithm would adopt crowding distance for further sorting, which requires $N\geq 4(n+1)$ such that the objective vector closest to the reference point can be maintained in the population. However, R-NSGA-II adopts preference information to further sort the solutions in $F_1$. Since the preference information here just means the distance to the reference point, the objective vector closest to the reference point can always be maintained regardless of the population size $N$. Following the analysis of Theorem~1, the expected number of iterations for finding the reference point is $O(Nn \log n)$.
\end{proof}
From the proofs of Theorems~\ref{theo-omm-nsgaii} and~\ref{theo-omm-rnsgaii}, we can find why R-NSGA-II can be better than NSGA-II. NSGA-II uses crowding distance to rank solutions, which has no bias towards the reference point. Thus, a large population size is required to ensure that each objective vector obtained will not be lost. Then, by continuously expanding in the Pareto front, the reference point can be found. However, R-NSGA-II prefers the solution closer to the reference point, and thus can always drive the population towards the reference point regardless of the population size. 

\subsection{Analysis on OneJumpZeroJump}

The OneJumpZeroJump problem as presented in Definition~4 is a bi-objective variant of the single-objective problem $\text{JUMP}_{n,k}$~\cite{droste2002analysis}. Its first objective is the same as $\text{JUMP}_{n,k}$, which aims at maximizing the number of 1-bits except a valley around $1^n$. This valley necessitates the flipping of precisely $k$ correct bits to traverse it from the local optimum. The second objective shares isomorphism with the first one, wherein the roles of 0s and 1s are exchanged. The Pareto set is $S^* = \{ \bmx\in \{0,1\}^n\mid |\bmx|_1\in [k..n-k]\cup\{0,n\}\}$, and the Pareto front is $F^* = \{(i, n+2k-i)\mid i\in[2k..n]\cup \{k, n+k\}\}$ whose size is $|F^*|=n-2k+3$. We let the reference point $\bmf(\bmz) = (n+k, k)$ to see if combining interactive decision making into NSGA-II can speed up the process of traversing the valley and finding the extreme objective vector in the Pareto front. Note that the solution corresponding to the reference point $(n+k, k)$ is the solution with all 1s.

\begin{definition}[OneJumpZeroJump~\cite{doerr2021theoretical}]
    The OneJumpZeroJump problem of size $n$ is to find $n$-bits binary strings which maximize
    \begin{align}
        f_1(\bmx) = \begin{cases}
k+|\bmx|_1, &\text{if }|\bmx|_1\leq n-k\ \text{or }\bmx=1^n,\\
n-|\bmx|_1, &\text{else;}
\end{cases}
\\
f_2(\bmx) = \begin{cases}
k+|\bmx|_0, &\text{if }|\bmx|_0\leq n-k\ \text{or }\bmx=0^n,\\
n-|\bmx|_0, &\text{else;}
\end{cases}
    \end{align}
    where $n \in \mathbb{N}$, $k\in[2..n/4]$, and $|\bmx|_1$ and $|\bmx|_0$ denote the number of 1-bits and 0-bits of $\bmx$, respectively.
\end{definition}
For the convenience of proof, we refer to the notation in~\cite{doerr2022first}, i.e., the inner part of the Pareto set is denoted as
\begin{align}
    S_{I}^*=\{\bmx\mid |\bmx|_1\in[k..n-k]\},
\end{align}
and the outer part of the Pareto set is denoted as
\begin{align}
    S_{O}^{*}=\{\bmx\mid |\bmx|_1\in \{0,n\}\};
\end{align}
the inner part of the Pareto front is denoted as
\begin{align}
    F_{I}^{*}=\bmf(S_{I}^*)=\{(i,n+2k-i)\mid i\in[2k..n]\},
\end{align}
and the outer part of the Pareto front is denoted as
\begin{align}
    F_O^*=\bmf(S_O^*)=\{(i,n+2k-i)\mid i\in\{k,n+k\}\}.
\end{align}

We prove in Theorems~\ref{theo-ojzj-nsgaii} and~\ref{theo-ojzj-rnsgaii} that NSGA-II can find the reference point in $O(Nn^k)$ expected number of fitness evaluations, where $N\geq 4|F^*|=4(n-2k+3)$, while R-NSGA-II using any population size $N \in o(n^k)$ can find the reference point in $O(n^k)$ expected number of fitness evaluations. Thus, R-NSGA-II can be $n$ times faster than NSGA-II. The proof of Theorem~\ref{theo-ojzj-nsgaii} is almost the same as that of Theorem~6 in~\cite{doerr2022first}. Compared with finding the whole Pareto front, finding a reference point, i.e., a single objective vector in the Pareto front, can be easier. However, our setting $(n+k, k)$ of the reference point is actually the most difficult one for OneJumpZeroJump, making that the expect running time for finding the reference point is the same as that for finding the whole Pareto front. 
\begin{theorem}\label{theo-ojzj-nsgaii}
    For NSGA-II solving the OneJumpZeroJump problem, if using a population size $N$ of at least $4(n+2k-3)$, then the expected number of fitness evaluations for finding the reference point $(n+k, k)$ is $O(Nn^k)$.
\end{theorem}
\begin{proof}
    According to the analysis in~\cite{doerr2022first}, we divide the run of NSGA-II optimizing the OneJumpZeroJump benchmark into the following stages.  
    \begin{itemize}
        \item Stage 1: this stage starts after initialization and ends when at least one solution belonging to $S_I^*$ is found.
        \item Stage 2: this stage starts  after stage~1 and ends when the entire $F_I^*$ is covered.
        \item Stage 3: this stage after stage~2 and ends when the entire $F^*$ is covered.
    \end{itemize}
    The expected number of fitness evaluations of these three stages is $Ne(4k/3)^k$, $O(Nn\log n)$ and $O(Nn^k)$, respectively~\cite{doerr2022first}. Note that the goal of stage~3 is to find the extreme objective vectors $(n+k,k)$ and $(k, n+k)$, which is the most time-consuming part and also coincides with the goal of finding the reference point. Thus, the expected number of fitness evaluations for finding the reference point $(n+k, k)$ is $O(Nn^k)$.
\end{proof}

The basic proof idea of Theorem~\ref{theo-ojzj-rnsgaii} is similar to that of Theorem~\ref{theo-ojzj-nsgaii}, i.e., dividing the optimization procedure into three phases. The main difference is that during the optimization procedure, the population of R-NSGA-II will first converge to the objective vector $(n, 2k)$, and then jump to the reference point $(n+k, k)$.
\begin{theorem}\label{theo-ojzj-rnsgaii}
    For R-NSGA-II with any population size $N \in o(n^k)$ solving the OneJumpZeroJump problem, the expected number of fitness evaluations for finding the reference point $(n+k, k)$ is $O(n^k)$.
\end{theorem}
\begin{proof}
    We divide the optimization procedure into three stages. Note that we pessimistically assume that the reference point is not found in stages~1 and~2.
    \begin{itemize}
        \item Stage 1: this stage starts after initialization and ends when the solution with $(n-k)$ 1-bits is maintained in the population.
        \item Stage 2: this stage starts after stage~1 and ends when the population is full of the solutions with $(n-k)$ 1-bits.
        \item Stage 3: this stage starts after stage~2 and ends when the reference point is found, i.e., the solution $1^n$ is maintained in the population.
    \end{itemize}

    \textbf{Stage 1.} We first analyze the expected number of generations until finding a solution in $S_I^*$. To this end, we can pessimistically assume that any solution in the initial population has at most $(k-1)$ 1-bits or at least $(n-k+1)$ 1-bits. 
    Without loss of generality, we can assume that one solution $\bmx$ in the population has at most $(k-1)$ 1-bits. Then, selecting $\bmx$ and flipping $(k-|\bmx|_1)$ 0-bits can generate a solution in $S_I^*$. Note that each solution in the population will be selected for mutation in each generation, thus the probability of generating a solution in $S_I^*$ is at least 
	$$\frac{\binom{n-|\bmx|_1}{k-|\bmx|_1}}{n^{k-|\bmx|_1}}\cdot \Big(1-\frac{1}{n}\Big)^{n-k+|\bmx|_1}\ge \frac{1}{e}\cdot \frac{\binom{n-|\bmx|_1}{k-|\bmx|_1}}{n^{k-|\bmx|_1}}\ge \frac{1}{ek^k},$$
    where the last inequality is by the proof of Theorem~1 in~\cite{bian2023sms}. Note that the solution in $S_I^*$ dominates any non-Pareto optimal solution, and has a smaller distance to the reference point compared with $0^n$, thus the newly generated solution must be maintained in the population.
    Thus, the expected number of generations for finding a solution in $S_I^*$ is at most $e k^k$. 

    Next, we analyze the expected number of generations until finding a solution with $(n-k)$ 1-bits.
    By flipping a 0-bit of the solution with the most 1-bits in $S_I^*$, a solution in $S_I^*$ with smaller distance to the reference point can be generated, whose probability is at least $\frac{1}{n}\cdot (1-\frac{1}{n})^{n-1}\geq \frac{1}{en}$. Note that the new solution is the most preferred and thus will be maintained in the population. Repeating the above process at most $(n-2k)$ times, the solution with $(n-k)$ 1-bits can be maintained in the population. Thus, the total expected number of generations of stage~1 is at most $e k^k+(n-2k)\cdot en$.
      
    
    \textbf{Stage 2.} Note that any non-Pareto optimal solution must be dominated by the solution with $(n-k)$ 1-bits, and any Pareto optimal solution (except $1^n$) has smaller $f_1$ value and larger $f_2$ value than the solution with $(n-k)$ 1-bits; thus the solution with $(n-k)$ 1-bits is the most preferred. After a solution with $(n-k)$ 1-bits is found in stage~1, it can consecutively copy itself by flipping none of the bits, and occupy the whole population. The probability of the copying event is at least $(1-\frac{1}{n})^n\geq (1-\frac{1}{n})\cdot \frac{1}{e}\ge \frac{1}{6}$. As the duplication procedure needs to repeat at most $N$ times, the expected number of generations of stage~2 is at most $6N$.
        
    \textbf{Stage 3.} For a solution $\bmx$ with $|\bmx|_1=n-k$,  the flipping of the $k$ 0-bits can generate the reference point, whose probability is $(\frac{1}{n})^k\cdot(1-\frac{1}{n})^{n-k}\geq \frac{1}{en^k}$. Since all the $N$ solutions in the population have $(n-k)$ 1-bits, the probability of generating the reference point in each generation is at least $1-(1-\frac{1}{en^k})^N$. Thus, the expected number of fitness evaluations is at most $N/(1-(1-\frac{1}{en^k})^N)=O(n^k)$, where the equality holds by $N=o(n^k)$.

    In conclusion, the whole process takes $O(n^k)$ expected number of fitness evaluations to find the reference point.
\end{proof}

From the proofs of Theorems~\ref{theo-ojzj-nsgaii} and~\ref{theo-ojzj-rnsgaii}, we can find the main reason for the acceleration by R-NSGA-II. NSGA-II requires $N\geq 4(n-2k+3)$ to ensure that every objective vector in the Pareto front obtained by the algorithm will not be lost in the future. However, R-NSGA-II can drive the whole population to converge towards the reference point and thus increase the probability of jumping across the valley, which can eliminate the influence of the population size.  

The above comparison on OneMinMax and OneJumpZeroJump discloses two advantages of R-NSGA-II. First, due to the preference to the reference point, it can automatically maintain the solution closest to the reference point in the objective space, without requirement on the population size $N$. Second, it can drive the whole population to converge towards the reference point, bringing acceleration. For example, on OneJumpZeroJump, all solutions in the population will keep approaching the reference point, which can significantly increase the probability of leaping to the optima. 

\begin{figure*}[htbp]\centering
\begin{minipage}[c]{0.24\linewidth}
\centering
    \includegraphics[width=1\linewidth]{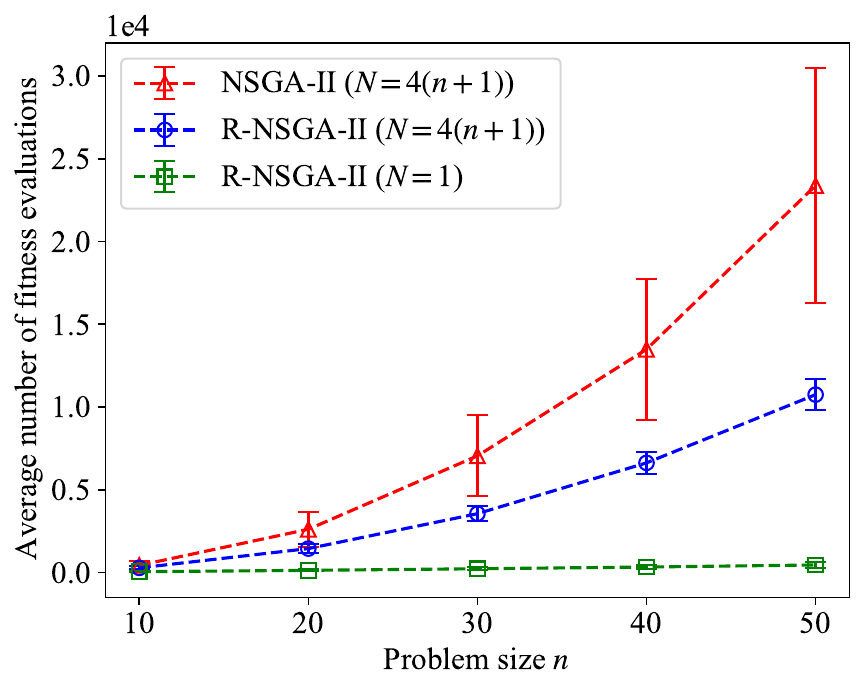}
\end{minipage}
\begin{minipage}[c]{0.24\linewidth}
\centering
    \includegraphics[width=1\linewidth]{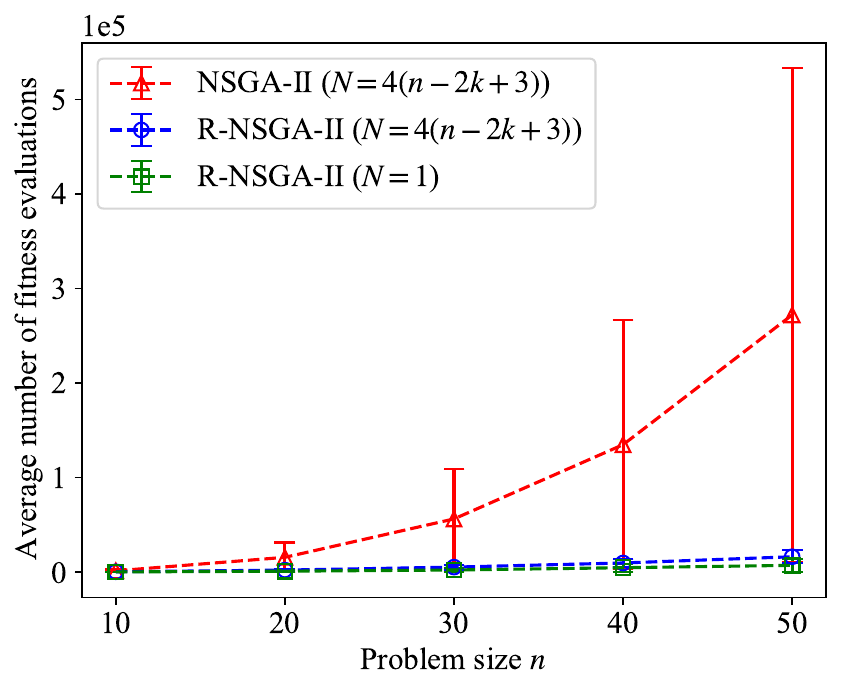}
\end{minipage}
\begin{minipage}[c]{0.24\linewidth}
\centering
    \includegraphics[width=1\linewidth]{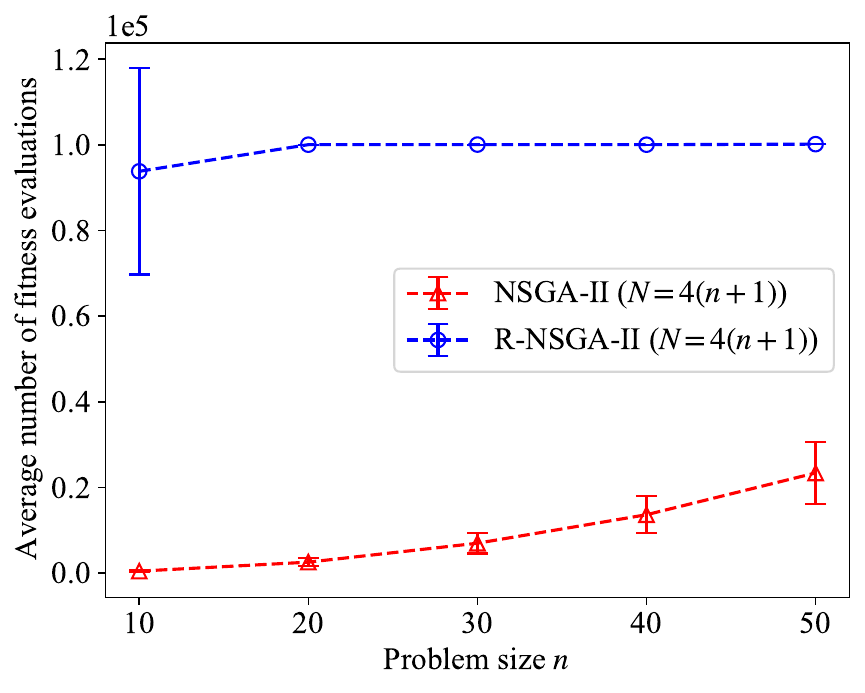}
\end{minipage}\vspace{0.1em}
\begin{minipage}[c]{0.24\linewidth}
\centering
    \includegraphics[width=1\linewidth]{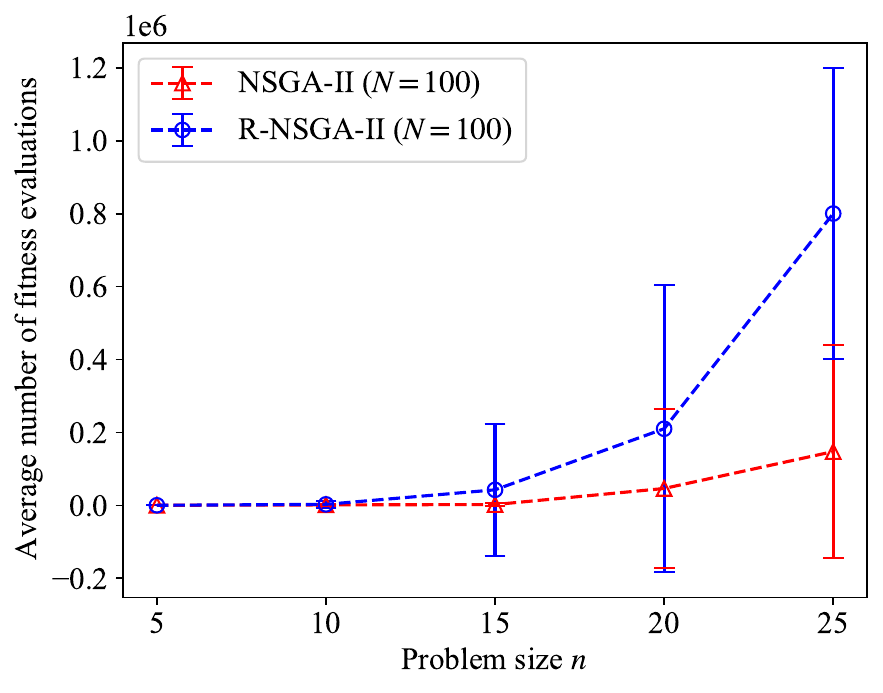}
\end{minipage}\\
\begin{minipage}[c]{0.24\linewidth}\centering
    (a) OneMinMax
\end{minipage}\ \
\begin{minipage}[c]{0.24\linewidth}\centering
    (b) OneJumpZeroJump
\end{minipage}
\begin{minipage}[c]{0.24\linewidth}\centering
    (c) OneMinMax*
\end{minipage}
\begin{minipage}[c]{0.24\linewidth}\centering
    (d) NK-Landscape
\end{minipage}
\caption{Average number of fitness evaluations of 1000 independent runs of NSGA-II and R-NSGA-II for solving the OneMinMax, OneJumpZeroJump, OneMinMax* and bi-objective NK-landscape problems, respectively.}\label{fig:artificial}
\end{figure*}

\section{Can R-NSGA-II Always Work?}

In the previous section, we have shown the advantage of R-NSGA-II theoretically, which echoes what people said, i.e., combining decision making into MOEAs can accelerate finding the DM's favorite solution. However, we will show that R-NSGA-II may fail sometimes. For this purpose, we consider a variant of OneMinMax, called OneMinMax* as in Definition~\ref{def-OneMinMax*}. The only difference is that the objective vector of the solution $0^n$ is set to $(-n, 2n)$. Now, the Pareto front is $F^* = \{(k, n-k)\mid k\in[0..n-1]\}\cup\{(-n, 2n)\}$, and any solution is still Pareto optimal. We set the reference point to be $\bmf(\bmz)=(-n,2n)$, i.e., $\bmz=0^n$.

\begin{definition}[OneMinMax*]\label{def-OneMinMax*}
    The OneMinMax* problem of size $n$ is to find $n$-bits binary strings which maximize
    \begin{equation}
    \begin{aligned}
        \bmf(\bmx) \!= \!\begin{cases}
(\sum_{i=1}^{n}(1-x_i),\sum_{i=1}^{n}x_i), &\text{if }\bmx\!\in\!\{0,1\}^n\setminus\{0^n\},\\
(-n, 2n),&\text{if }\bmx=0^n.
\end{cases}
    \end{aligned}
    \end{equation}
\end{definition}

We prove in Theorems~\ref{theo-omm*-nsgaii} and~\ref{theo-omm*-rnsgaii} that NSGA-II with $N\geq 4(n+1)$ can find the reference point in $O(Nn\log n)$ expected number of fitness evaluations, while R-NSGA-II using any polynomial population size requires $\Omega(n^{n/4})$ expected running time. Thus, R-NSGA-II can be exponentially slower than NSGA-II. The proof of Theorem~\ref{theo-omm*-nsgaii} is the same as that of Theorem~\ref{theo-omm-nsgaii}, because NSGA-II can find any reference point by continuously expanding in the Pareto front.

\begin{theorem}\label{theo-omm*-nsgaii}
    For NSGA-II solving the OneMinMax* problem, if using a population size $N$ of at least $4(n+1)$, then the expected number of fitness evaluations for finding the reference point $(-n,2n)$ is $O(Nn\log n)$.
\end{theorem}

R-NSGA-II drives the whole population to converge towards the reference point in the objective space. However, on the OneMinMax* problem, a solution closer to the reference point $(-n,2n)$ in the objective space will be farther from the corresponding solution $0^n$ in the decision space. This leads to the inefficiency of R-NSGA-II, which is also the main proof idea of Theorem~\ref{theo-omm*-rnsgaii}. 

\begin{theorem}\label{theo-omm*-rnsgaii}
    For R-NSGA-II with any polynomial population size $N$ solving the OneMinMax* problem, the expected number of fitness evaluations for finding the reference point $(-n,2n)$ is $\Omega(n^{n/4})$.
\end{theorem}
\begin{proof}
    By Chernoff bound, an initial solution randomly selected from $\{0,1\}^n$ has at least $n/4$ 1-bits with probability $1-e^{-\Omega(n)}$. Thus, the probability that any individual in the initial population has at least $n/4$ 1-bits is at least $(1-e^{-\Omega(n)})^{N}\geq 1-N e^{-\Omega(n)}=1-o(1)$, where the equality holds because $N$ is a polynomial of $n$.

    Now, we assume that the above event happens, i.e., any solution in the initial population has at least $n/4$ 1-bits. Note that any solution in the solution space is Pareto optimal, thus it will belong to $F_1$ after non-dominated sorting. Meanwhile, since the reference point is $(-n, 2n)$, the solution with smaller $f_1$ and larger $f_2$ is preferred. Thus, during the optimization procedure, the solution with more 1-bits must be preferred, implying that the population can maintain only solutions with at least $n/4$ 1-bits. To find the reference point, i.e., the solution $0^n$, the 1-bits of a parent solution must be flipped simultaneously. As each solution has at least $n/4$ 1-bits, the probability is at most $1/n^{n/4}$. Thus, in each generation, the probability of finding $0^n$ is at most 
    \begin{align}
    1-\Big(1-1/n^{n/4}\Big)^N\leq N/n^{n/4},
    \end{align}
    where the inequality is by Bernoulli's inequality. 

    Combining the above analyses, the expected number of fitness evaluations for finding the reference point is at least $(1-o(1))\cdot (n^{n/4}/N)\cdot N=\Omega(n^{n/4})$, where the factor $N$ corresponds to $N$ fitness evaluations in each generation.
\end{proof}

\section{Experiments}

In the previous sections, we have proved that combining interactive decision making into NSGA-II can bring acceleration for solving the OneMinMax and OneJumpZeroJump problems, while leading to significant deceleration for solving OneMinMax*. In this section, we first conduct experiments to verify these theoretical results, and then show that the finding from the analysis of OneMinMax* also holds on the combinatorial problem multi-objective NK-landscape~\cite{aguirre2007working}. The data and code are provided in the supplementary material.

For each problem OneMinMax, OneJumpZeroJump or OneMinMax*, we compare the number of fitness evaluations of NSGA-II and R-NSGA-II until finding the reference point, which is set to $(0,n)$, $(n+k,k)$ and $(-n,2n)$, respectively, consistent with our theoretical analysis. The first two reference points correspond to the solution with all 1s, and the last one corresponds to the solution with all 0s. We set the problem size $n$ from 10 to 50, with a step of 10. For each problem size $n$ and each algorithm, we conduct 1000 independent runs, and record the average and standard deviation of the number of fitness evaluations required to find the reference point.

For the OneMinMax problem, we set the population size $N$ of NSGA-II and R-NSGA-II to $4(n+1)$ and $1$, respectively, as suggested by Theorems~\ref{theo-omm-nsgaii} and~\ref{theo-omm-rnsgaii}. Figure~\ref{fig:artificial}(a) shows that R-NSGA-II with $N=1$ can achieve a clear acceleration compared with NSGA-II. We also test the performance of R-NSGA-II with the population size $N=4(n+1)$ to see whether the acceleration is only brought by the decrease of the population size. We can observe that even using the same population size, R-NSGA-II is still faster. 

For the OneJumpZeroJump problem, the parameter $k$ is set to $2$. The population size $N$ of NSGA-II and R-NSGA-II is set to $4(n-2k+3)$ and $1$, respectively, according to Theorems~\ref{theo-ojzj-nsgaii} and~\ref{theo-ojzj-rnsgaii}. We can see from Figure~\ref{fig:artificial}(b) that R-NSGA-II with $N=1$ needs much less fitness evaluations than NSGA-II with $4(n-2k+3)$, to find the reference point. Moreover, by increasing the population size $N$ to $4(n-2k+3)$, the average time of R-NSGA-II only increases slightly, which is consistent with Theorem~\ref{theo-ojzj-rnsgaii} that the expected running time of R-NSGA-II does not depend on $N$.

For the OneMinMax* problem, we set the population size $N$ of NSGA-II to $4(n+1)$ as suggested by Theorem~\ref{theo-omm*-nsgaii}, and use the same $N$ for R-NSGA-II for fairness. As we find that it's very time-consuming for R-NSGA-II to find the reference point, we set a maximum number of fitness evaluation, i.e., $10^5$. Figure~\ref{fig:artificial}(c) shows that NSGA-II is clearly better, and R-NSGA-II almost reaches the maximum time limit, which is consistent with Theorem~\ref{theo-omm*-rnsgaii} that R-NSGA-II requires at least exponential time to find the reference front.

From the proof of Theorem~\ref{theo-omm*-rnsgaii}, we find that for the OneMinMax* problem, a solution closer to the reference point in the objective space is farther from it in the decision space, leading to the inefficiency of R-NSGA-II. Such characteristic may hold in complex real-world problems, where R-NSGA-II can be inefficient. To validate this finding, we compare R-NSGA-II with NSGA-II on the bi-objective NK-Landscape problem, where the problem size $n$ is set from 5 to 25, with a step of 5. For each $n$, we choose a set of randomly initialized parameters, i.e., fitness contribution $c_{ij}$ and loci $k_{ij}$, to generate the problem, following~\cite{li2023moeas}. The parameter $K$ is set to 3, and the population size $N$ of NSGA-II and R-NSGA-II is both set to 100 for fairness. We acquire the Pareto front by exhaustive enumeration over the whole solution space $\{0,1\}^n$, and then randomly choose one from the Pareto front as the reference point, which is $[0.821, 0.483]$, $[0.643,0.647]$, $[0.770,0.481]$, $[0.741,0.521]$ and $[0.760,0.541]$ for the problem size $n \in \{5,10,\ldots,25\}$, respectively. For R-NSGA-II, we set a maximum number of fitness evaluations, i.e., $10^6$. Figure~\ref{fig:artificial}(d) shows that R-NSGA-II is much worse than NSGA-II.


\section{Conclusion}

In this paper, we make a first step towards theoretically analyzing the expected running time of practical iMOEA. We consider the interactive variant, R-NSGA-II, of the popular NSGA-II. We prove that R-NSGA-II can be asymptotically faster than NSGA-II for solving the OneMinMax and OneJumpZeroJump problems, providing theoretical justification for iMOEA. Furthermore, we find that the integration of interactive decision-making within MOEAs does not always expedite convergence by proving that R-NSGA-II can be exponentially slower than NSGA-II for solving a variant of the OneMinMax problem. Though being proved in a special case, this finding may hold more generally, especially on complex real-world problems with rugged fitness landscape. We have empirically verified it on the popular practical problem NK-landscape. This finding suggests the importance of striking a balance between interactive decision-making and diversity preservation within the population, which might be helpful for designing efficient iMOEAs in practice. 

\bibliography{main}

\appendix

\end{document}